\documentclass[twoside,11pt]{jmlr}

\usepackage{amsmath,amssymb,mathtools,bm}
\usepackage{newtxtext,newtxmath}
\usepackage{booktabs,array,multirow}
\usepackage{enumitem}
\usepackage{graphicx}
\usepackage{microtype}
\usepackage{xcolor}
\usepackage{url}

\allowdisplaybreaks
\setlist[itemize]{leftmargin=1.35em,itemsep=2pt,topsep=4pt}
\setlist[enumerate]{leftmargin=1.6em,itemsep=2pt,topsep=4pt}
\newcolumntype{P}[1]{>{\raggedright\arraybackslash}p{#1}}

\newcommand{\cF}{\mathcal{F}}

\newcommand{\cC}{\mathcal{C}}

\newcommand{\bbF}{\mathbb{F}}
\newcommand{\bbR}{\mathbb{R}}
\newcommand{\im}{\operatorname{im}}
\newcommand{\rank}{\operatorname{rank}}
\newcommand{\dist}{\operatorname{dist}}

\title[Capability Sheaves for Agent-Harness Repair]{Capability Sheaves for Compositional Agent-Harness Repair:\\Controlled Quotients and a Real-Repository Stress Test}

\author[Batruin]{%
  \Name{Saveliy Batruin}\Email{saveliy.xo@gmail.com}\\
  \addr Independent researcher
}

\jmlrproceedings{}{Preprint}
\jmlrpages{}
\jmlryear{2026}
\hypersetup{pdftitle={Capability Sheaves for Compositional Agent-Harness Repair},pdfauthor={Saveliy Batruin},pdfsubject={Finite capability sheaves, hidden-state quotients, and real-repository patch fusion},pdfkeywords={AI agents, patch fusion, sheaf cohomology, global sections, constraint satisfaction, SWE-bench}}

\begin{document}
\maketitle

\begin{abstract}
Agent harnesses combine retrieval, routing, state, provenance, and
verification, but locally successful components may disagree on shared state.
We model this failure with a finite \emph{capability sheaf}: stalks encode typed
behavior signatures, restriction maps retain shared fields, and accepted runs
are useful global sections. An exact finite constraint-satisfaction problem
(CSP) defines acceptance, while a linearized relative cohomology class provides
a diagnostic and search feature.

A controlled experiment over 20 task clusters introduces hidden interior
mediators whose raw states are nuisance variables. Quotienting their
coboundaries reduces the candidate budget from 2,000 to 1,000 per cluster;
aligning the hidden state removes the gap. Exact CSP matches the quotient, so
the result demonstrates invariance to stale representatives, not superiority
over exact reasoning.

We then test the method on a discovery split from the SWE-bench Multilingual
pool of PatchFuseBench: 160 issues from 20 repositories, 875 real candidate
patches, 2,579 source-aware edit atoms, and 153 newly executed patches. A first
pool-level construction is constant because $[b-Dx]=[b]$ in
$\operatorname{coker}D$ and therefore cannot rank configurations. A
candidate-indexed repair is nontrivial on 848/875 candidates and varies within
120/160 issues. It resolves 118 issues versus 116 for a matched
noncohomological selector, but the difference is not supported across
repositories (exact sign-flip $p=0.75$). A leave-one-repository-out abstention
gate reaches 127/160, tying the strong anchor and exceeding its matched gate by
one issue ($p=1.0$). The discovery gate therefore fails and the confirmatory
split remains sealed. The study supports the controlled invariance mechanism
and an identifiability correction, but not a real-world cohomological
advantage.
\end{abstract}


\section{Introduction}
\label{sec:intro}

The behavior of a language-model agent depends on its harness: instructions,
retrieval, routing, state, provenance, verification, and recovery. Current
systems optimize prompts, programs, workflows, or agent designs through
outer-loop search
\citep{khattab2023dspy,hu2025adas,zhang2025aflow,lee2026metaharness,
ursekar2026vero,chen2026harnessfix}. Scalar scores rank complete candidates.
They do not explain a common failure: each required capability appears locally
available, but no execution combines the capabilities coherently.

Consider a repository patch agent. One subsystem can locate the owned file,
another can recover its declared API revision, a third can preserve the source
commit, and a verifier can select a public contract test. These answers may
refer to different files or revisions. Each subsystem looks competent; their
shared path, revision, commit, and test state do not glue. This is a
local-to-global failure of executable behavior.

Sheaf theory gives precise language for local data, shared restrictions,
agreement, and gluing \citep{bredon1997sheaf,maclane1992sheaves,curry2014sheaves}.
Here the objects are finite and executable. We use typed stalks and literal
field projections. An exact CSP decides semantic feasibility. Linear
cohomology provides a diagnostic and a search score. This separation matters:
a cohomological witness can be useful without fully deciding whether an
executable candidate exists
\citep{abramsky2011sheaf,abramsky2012cohomology,caru2017cohomology}.

The paper makes four contributions.
\begin{enumerate}
  \item We construct a finite capability sheaf for five requirements from
  disjoint typed traces. It includes an exact CSP, relative classes, and
  independently checked restriction maps.
  \item We prove exact gluing, local-score separation, spectral stability, and
  finite-trace repair recovery, while giving counterexamples to invalid
  converses of the linear relaxation.
  \item We report a controlled hidden-state experiment across 20 independent
  task clusters. Its aligned-state ablation tests the proposed invariance
  mechanism directly.
  \item We run a real patch-fusion stress test on 20 repositories. It includes
  official execution of new patches, clustered inference, an explicit
  identifiability counterexample, and a candidate-indexed repair.
\end{enumerate}

The claim is narrower than a general optimizer result. Exact CSP is the
decisive control. In the controlled task family, quotienting removes a nuisance
interior representative. In the real benchmark, the corrected quotient is
configuration-specific but does not pass the development gate. We report both
results because the failure marks the current boundary of the method.

\section{Capability Sheaves}
\label{sec:model}

\subsection{Finite behavior stalks and restrictions}

Let $X$ be a finite incidence graph whose vertices are the five registered
requirements
\[
L=\text{localization},\quad C=\text{contract},\quad O=\text{ordering},
\quad P=\text{preservation},\quad V=\text{verification}.
\]
A cellular sheaf $\cF$ assigns a finite behavior-signature stalk $\cF(v)$ to
each vertex, an overlap stalk $\cF(e)$ to each edge, and a restriction
$\rho_{v\to e}:\cF(v)\to\cF(e)$ for every incidence. The vertex signatures are
\begin{center}
\begin{tabular}{ll}
\toprule
requirement & typed signature fields\\
\midrule
localization & file path, namespace alias, symbol\\
contract & file path, API revision, source commit\\
ordering & API revision, edit order, test identifier\\
preservation & file path, source commit, namespace alias\\
verification & file path, symbol, test identifier\\
\bottomrule
\end{tabular}
\end{center}
The learned sparse cover has six overlaps: $L$--$V$ shares file and symbol;
$L$--$P$ shares file and namespace; $C$--$P$ shares file and commit;
$O$--$V$ shares the test identifier; $P$--$V$ shares file; and $C$--$O$
shares API revision. Every restriction is literal field projection.

For one target in the baseline complex, a boundary behavior is a tuple
$s_A=(s_v)_{v\in X^{(0)}}$ on the vertex subcomplex $A=X^{(0)}$. Each vertex has a registered good subset
$G_v\subseteq\cF(v)$. Compatibility and local usefulness are distinct:
\[
s_v\in G_v\quad\text{and}\quad
\rho_{v\to e}(s_v)=\rho_{w\to e}(s_w)\ \text{for }e=vw.
\]

\subsection{Exact CSP and gluing}

For a candidate harness $c$, let $s_v(c)$ be its selected local behavior. The
exact feasibility predicate is
\begin{equation}
\label{eq:exact}
\Phi(c)=
\bigwedge_{v\in X^{(0)}}[s_v(c)\in G_v]
\ \land\!
\bigwedge_{e=vw\in X^{(1)}}
[\rho_{v\to e}s_v(c)=\rho_{w\to e}s_w(c)].
\end{equation}
For finite stalks this is an ordinary finite CSP. The sheaf formulation does
not remove combinatorial hardness; it separates local existence from gluing
and supplies typed linear diagnostics across candidates.

\begin{theorem}[Augmented feasibility]
\label{thm:gluing}
Let $\{U_i\}$ be a cover of a task domain and let $\cF$ be a sheaf of feasible
behaviors on that cover.  Suppose (i) the local predicate for $U_i$ is true
exactly when the selected section $s_i\in\cF(U_i)$ exists and is useful, and
(ii) every overlap needed for the sheaf gluing axiom is registered.  Then a
candidate represents a global section accepted by the cover-local predicate
in Equation~\eqref{eq:exact} if and only if every local predicate holds and
every registered pair of restrictions agrees.
\end{theorem}

\begin{proof}
If all local predicates hold, assumption (i) supplies the selected local
sections.  Assumption (ii) turns the overlap bits into equality of every pair
of restrictions needed by the cover.  The sheaf gluing axiom gives a unique
global section with those restrictions, and this section is accepted by the
conjunction in Equation~\eqref{eq:exact}.  Conversely, restrictions of one
accepted global section exist locally, satisfy the registered predicates by
assumption (i), and agree on every overlap.  Both assumptions are necessary:
missing local sections can make compatibility vacuous, and an omitted overlap
can hide a conflict.
\end{proof}

\subsection{Relative connecting obstruction}

Represent each finite field value by an exact one-hot basis vector over
$\bbF_2$. The cellular coboundary is
\begin{equation}
\label{eq:delta}
(\delta^0 x)_e=\rho_{v\to e}x_v+\rho_{w\to e}x_w,
\qquad e=vw.
\end{equation}
The short exact sequence of the pair $(X,A)$ induces a connecting map, and the
boundary section has class
\begin{equation}
\label{eq:relative}
\partial[s_A]=[\delta^0s_A]\in H^1(X,A;\cF).
\end{equation}
Because $A$ contains every vertex, $C^0(X,A;\cF)=0$; hence this class vanishes
exactly when adjacent typed restrictions agree. It is a genuine relative
cellular-sheaf class for the declared pair, although in this experiment it is
algebraically a structured edge-mismatch vector rather than a quotient over
latent interior completions.

The independently reconstructed one-hot complex has
\[
\dim C^0=246,\qquad \dim C^1=233,\qquad \rank\delta^0=173,
\]
so $\dim H^1(X;\cF)=60$ and $\dim H^1(X,A;\cF)=233$. These dimensions
describe the representation, not failure prevalence.

\subsection{Hidden interior states and a nontrivial quotient}
\label{sec:hidden-quotient}

The first experiment fixes $A=X^{(0)}$. It has no relative degree-zero
cochains, so its class is the raw boundary mismatch. The second experiment
changes this geometry. For every typed overlap coordinate $j$, we insert a
hidden mediator vertex between the two public requirement vertices. The public
vertices form the boundary $A$; the mediators lie in $X\setminus A$.

Let $V_j$ be the one-hot vector space for coordinate $j$. If the public values
are $\ell_j,r_j\in V_j$ and the hidden value is $h_j\in V_j$, one extension of
the boundary has relative residual
\begin{equation}
\label{eq:hidden-residual}
q_j(h_j)=(\ell_j+h_j,\;h_j+r_j)\in V_j\oplus V_j.
\end{equation}
Changing $h_j$ by $u$ adds the interior coboundary
\[
D_j u=(u,u),\qquad D_j=\begin{bmatrix}I\\ I\end{bmatrix}.
\]
Over $\bbF_2$, the map $P_j=[I\ I]$ satisfies $P_jD_j=0$ and
$\ker P_j=\im D_j$. Hence
\begin{equation}
\label{eq:hidden-quotient}
Q_j=(V_j\oplus V_j)/\im D_j\cong V_j,
\qquad [q_j(h_j)]\longmapsto \ell_j+r_j.
\end{equation}
The quotient has positive dimension $\dim V_j$, and its class does not depend
on the hidden state. Taking direct sums over coordinates gives
$C^1(X,A;\cF)/\im\delta^0_{\mathrm{int}}$. For actual one-hot endpoints, the
class is zero exactly when $\ell_j=r_j$ on every registered coordinate.

\begin{figure}[t]
\centering
\includegraphics[width=0.98\textwidth]{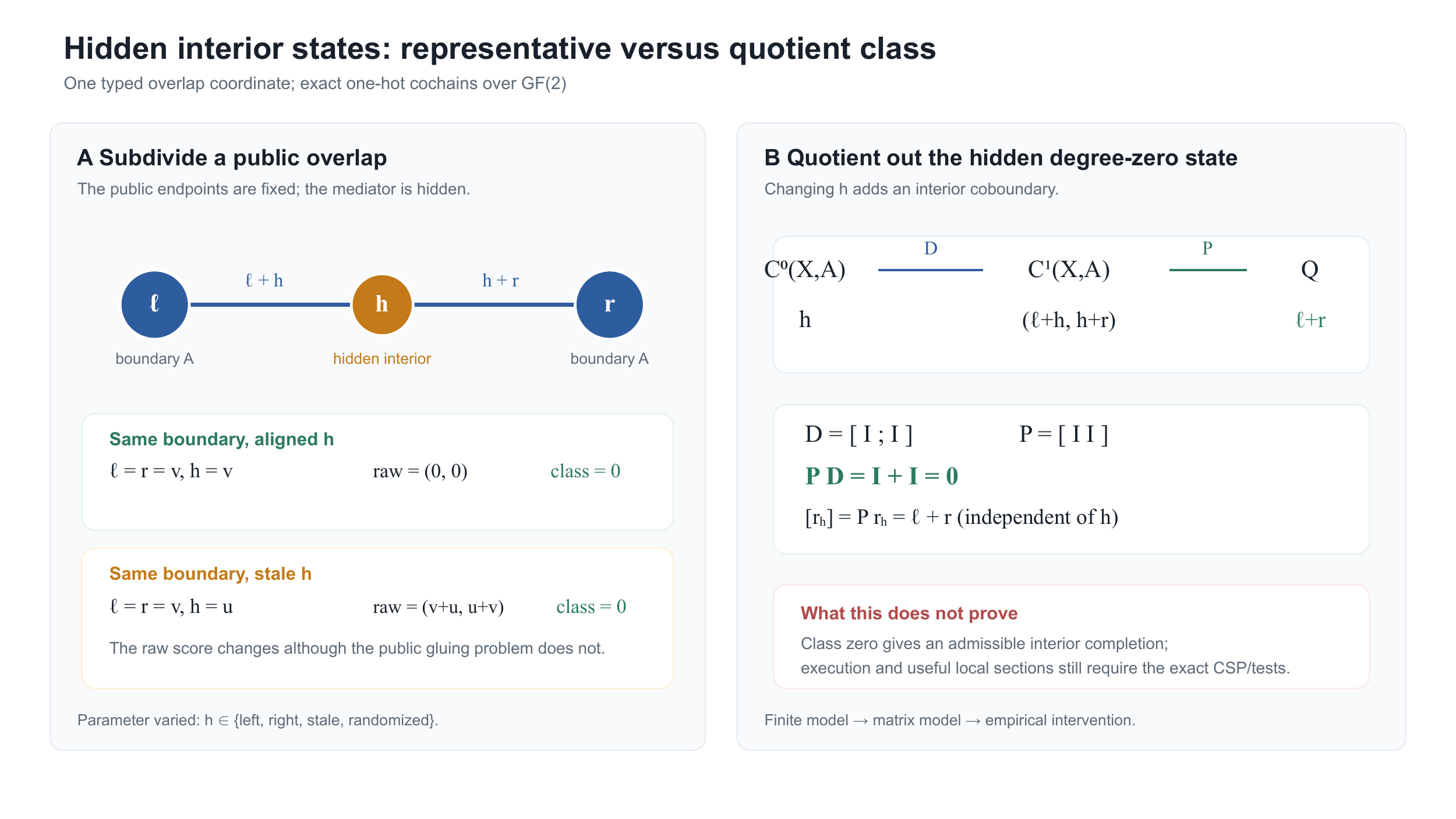}
\caption{A finite and matrix model of the hidden-state quotient. The raw
half-edge residual changes when the mediator state $h$ changes. The quotient
map removes this interior coboundary and keeps only the endpoint disagreement
$\ell+r$. The picture does not replace the exact CSP: class zero checks the
registered gluing fields, while useful local sections and execution are
separate requirements.}
\label{fig:hidden-design}
\end{figure}

This construction supports a direct intervention. We hold the public task and
candidate fixed and vary the mediator state: stale, aligned with the left
endpoint, aligned with the right endpoint, or randomized. The quotient ranking
should stay fixed. A raw half-edge score may change. The aligned condition is a
negative control: when $h_j=\ell_j$, the raw mismatch already equals the public
endpoint mismatch, so quotienting should not improve the stopping budget.

\subsection{What linearization does not prove}

For an actual one-hot boundary, Equation~\eqref{eq:relative} exactly tests the
registered compatibility fields. For a portfolio of executable candidates,
one may instead form signatures $\sigma(c)\in\bbF_2^d$, their span $S$, and a
quotient class $[t]\in\bbF_2^d/S$ for a target signature $t$.

\begin{proposition}[One-sided portfolio certificate]
\label{prop:onesided}
If $[t]\ne0$, then no candidate has signature $t$.  The converse is false.
\end{proposition}

\begin{proof}
An exact candidate would put $t$ in $S$, so a nonzero quotient is a sound
no-candidate certificate.  For the converse, take candidates with signatures
$(1,0)$ and $(0,1)$ and target $(1,1)$.  The target lies in their linear span,
but neither candidate realizes it.  Therefore quotient vanishing must be
followed by exact membership or execution.
\end{proof}

Quotient vanishing must therefore be followed by exact candidate membership or
execution. Likewise, compatibility cannot substitute for missing useful local
sections, which is why Equation~\eqref{eq:exact} retains both obligations.

\section{Repair and Search}
\label{sec:repair}

\subsection{Typed fillers and obstruction-guided completion}

The registered library contains five transformations:
\begin{itemize}
  \item a symbol-index router exposing symbol-to-file retrieval;
  \item a revision ledger exposing immutable file-to-commit provenance;
  \item a workspace partition exposing deployment-to-workspace routing;
  \item a dependency checkpoint exposing API revision and edit order; and
  \item a test-oracle verifier executing public tests with at most one
  correction call.
\end{itemize}
The common candidate pool consists of all nonempty subsets of size at most
three, hence 25 bundles. A filler changes typed evidence and the induced local
sections; after application we recompute both the exact CSP and the relative
class.

The outcome-blind completion operator enumerates the pool and lexicographically
minimizes relative obstruction weight, maximizes exact-lift fraction and
nomination margin, and then minimizes filler cost. It selects workspace
partition plus dependency checkpoint at cost 3 in every task replicate. Before
the proposal, the registered all-local structural decoy has nonzero relative
weight and no exact lift. After the proposal, relative weight is zero and the
exact-lift fraction is one. No confirmatory outcome enters this computation.

\subsection{Information separation}

\begin{theorem}[Local-score repair-order separation]
\label{thm:separation}
Let $c_1,\ldots,c_K$ have identical isolated local signatures.  Suppose exactly
one, $c_J$, has compatible restrictions and every other candidate has a
nonzero exact relative class.  A top-one selector whose observation is only
the isolated signature has minimax success at most $1/K$ under arbitrary
relabeling of the candidates.  A zero-class-first selector given the exact
restriction classes ranks $c_J$ first.  For any ordering independent of the
hidden feasible label, under a uniform random permutation the expected rank of
$c_J$ is $(K+1)/2$.
\end{theorem}

\begin{proof}
All $K$ candidates induce the same observation for the local-only selector.
Place the unique feasible index uniformly on the $K$ labels.  Conditional on
the common observation, any deterministic top-one choice succeeds with
probability $1/K$; randomization cannot improve this average, so Yao's minimax
principle yields the bound against an adversarial relabeling.  Exact relative
classes identify the unique zero-class candidate by hypothesis.  Finally, the
position of one distinguished element in a uniform permutation is uniform on
$\{1,\ldots,K\}$ and has mean $(K+1)/2$.
\end{proof}

The theorem identifies the information missing from isolated scores. It does
not imply that a cohomological relaxation outperforms an exact CSP supplied
with the same restrictions. The experiment therefore treats exact CSP as the
semantic control and adaptive NSGA-II, local marginal, pairwise, shuffled, and
random orderings as equal-pool search baselines.

\section{Stability of an Estimated Linear Obstruction}
\label{sec:stability}

Exact typed restrictions need no estimation.  Broader applications may learn
a linear coboundary $D\in\bbR^{p\times q}$ and target $b\in\bbR^p$ from fresh
traces.  Naively using the full image of a noisy matrix is invalid because an
arbitrarily small perturbation can increase rank.

\begin{theorem}[Rank-truncated residual stability]
\label{thm:stability}
Let $D$ have registered rank $r\ge1$, image $U$, and smallest nonzero singular
value $\gamma=\sigma_r(D)>0$.  Suppose
\[
\|\widehat D-D\|_2\le\varepsilon<\gamma,
\qquad \|\widehat b-b\|_2\le\tau,
\qquad \|b\|_2\le B.
\]
Let $\widehat U_r$ be the leading $r$-dimensional left singular subspace of
$\widehat D$, and put
\[
\rho=\min\left\{1,\frac{2\varepsilon}{\gamma-\varepsilon}\right\},
\qquad T=\tau+\rho B.
\]
Then
\[
\left|\dist(\widehat b,\widehat U_r)-\dist(b,U)\right|\le T.
\]
Consequently, thresholding the estimated residual at $T$ accepts every true
lift and rejects every population nonlift whose residual is greater than $2T$.
\end{theorem}

\begin{proof}
If $\rho=1$, the projector-distance bound is automatic for orthogonal
projectors.  Otherwise $\varepsilon<\gamma/3$; Weyl's inequalities give
$\sigma_r(\widehat D)\ge\gamma-\varepsilon$ and
$\sigma_{r+1}(\widehat D)\le\varepsilon$, so the leading $r$-subspace is
separated.  Wedin's singular-subspace perturbation theorem then bounds the
projector distance between $U$ and $\widehat U_r$ by the conservative value
$2\varepsilon/(\gamma-\varepsilon)=\rho$
\citep{wedin1972perturbation}.  With projectors $P$ and $\widehat P$,
\begin{align*}
\big|\|(I-\widehat P)\widehat b\|_2-\|(I-P)b\|_2\big|
&\le \|(I-\widehat P)(\widehat b-b)\|_2
   +\|(\widehat P-P)b\|_2\\
&\le \tau+\rho B=T.
\end{align*}
The two classification statements follow by applying this interval around a
population residual of zero or one strictly larger than $2T$.
\end{proof}

The deterministic statement gives the following finite-trace guarantee.

\begin{corollary}[Finite-trace lift/no-lift stability]
\label{cor:finite-trace-stability}
Suppose each of $n$ independent complete traces contributes unbiased
coordinate estimators of $D$ and $b$, with every coordinate supported on an
interval of length at most two.  Under the rank, gap, and norm hypotheses of
Theorem~\ref{thm:stability}, choose $\varepsilon<\gamma$ and $\tau>0$.  If
\begin{equation}
\label{eq:samples}
n\ge\max\left\{
\frac{2pq}{\varepsilon^2}\log\frac{4pq}{\delta},
\frac{2p}{\tau^2}\log\frac{4p}{\delta}
\right\}.
\end{equation}
then, with probability at least $1-\delta$, the rank-$r$ estimated residual
has error at most
$T=\tau+B\min\{1,2\varepsilon/(\gamma-\varepsilon)\}$.  Consequently the
threshold-$T$ decision preserves every population lift and every population
nonlift with residual strictly greater than $2T$.
\end{corollary}

\begin{proof}
Coordinatewise Hoeffding bounds at tolerances
$\varepsilon/\sqrt{pq}$ for $D$ and $\tau/\sqrt p$ for $b$, followed by a
union bound, give
$\|\widehat D-D\|_F\le\varepsilon$ and
$\|\widehat b-b\|_2\le\tau$ with probability at least $1-\delta$.
Since the operator norm is bounded by the Frobenius norm,
Theorem~\ref{thm:stability} applies on that event.
\end{proof}

Coordinates within one trace may be dependent; independence is used only
across complete traces.  Cached duplicates do not increase $n$.
Since the quotient obstruction $[b]\in\bbR^p/\im D$ vanishes exactly when
$\dist(b,U)=0$, Corollary~\ref{cor:finite-trace-stability} is a statistical
stability result for the lift/no-lift decision.  It does not assert equality
of two separated nonzero quotient classes.

Every displayed hypothesis matters.  Without rank truncation,
$D=\operatorname{diag}(1,0)$, $b=e_2$, and
$\widehat D=\operatorname{diag}(1,\varepsilon)$ produce a false lift for every
$\varepsilon>0$.  Without a positive gap the image is not stably identifiable;
without the strict margin the lift and nonlift residual intervals overlap.

\subsection{Finite-trace recovery of restriction-aware ordering}

The preceding theorem stabilizes a fixed linear obstruction.  A complementary
question is whether development traces recover the repair ordering itself.
Let $\cC=\{1,\ldots,K\}$ be a finite candidate pool and let
$\mu_c\in[0,1]^m$ be its population face-agreement vector.  Define the
weakest-face score $h(c)=\min_f\mu_{cf}$.
Here a face may encode either a registered local predicate or an overlap
agreement event, so $h$ is a canonical restriction-aware bottleneck score.
It is not the lexicographic policy frozen in the experiments; the theorem
isolates the sampling obligation for one explicit population objective.

\begin{theorem}[Finite-trace repair recovery]
\label{thm:recovery}
Assume $c^\star$ uniquely maximizes $h$ with gap
\[
\Delta=\min_{c\ne c^\star}\{h(c^\star)-h(c)\}>0.
\]
For every $c$, observe $n$ fresh vectors $Z_{c,t}\in[0,1]^m$ adapted to a
filtration, with fixed conditional mean
$\mathbb E[Z_{c,t}\mid\mathcal F_{t-1}]=\widetilde\mu_c$ and
$\|\widetilde\mu_c-\mu_c\|_\infty\le\eta<\Delta/2$.  Coordinates within one
trace may be dependent.  Put
$\widehat h(c)=\min_f n^{-1}\sum_t Z_{c,t,f}$.  If
\[
n\ge \frac{8}{(\Delta-2\eta)^2}\log\frac{2Km}{\delta},
\]
then $\arg\max_c\widehat h(c)=c^\star$ with probability at least
$1-\delta$.
\end{theorem}

\begin{proof}
Let $e=(\Delta-2\eta)/4>0$.  Coordinatewise conditional
Hoeffding--Azuma bounds \citep{hoeffding1963probability,azuma1967weighted}
and a union bound give
\[
\Pr\!\left(
\max_{c,f}|\widehat\mu_{cf}-\widetilde\mu_{cf}|>e
\right)
\le 2Km\exp(-2ne^2)\le\delta.
\]
On the complementary event,
$\|\widehat\mu_c-\mu_c\|_\infty\le e+\eta$.  Since the minimum is
one-Lipschitz in max norm, every rival satisfies
\[
\widehat h(c^\star)-\widehat h(c)
\ge \Delta-2(e+\eta)
=\frac{\Delta-2\eta}{2}>0.
\]
Thus the empirical maximizer is uniquely $c^\star$.
\end{proof}

The theorem permits arbitrary dependence among faces of the same trace but
requires freshness across complete traces.  Cached copies do not increase
$n$.  More importantly, it is an identification-conditional result:
concentration cannot turn a misspecified face vocabulary into a causal
restriction model.

\section{Trace Construction and Baseline Experiment}
\label{sec:experiment}

\subsection{Generated repository task}

Each task contains two migration tickets and four same-namespace adapter files
per ticket. Exactly one file is owned by the deployment. A valid response must
edit that file, install the dependency-declared API revision, place migration
before caller, preserve the owned source commit, and select the public contract
test. The model emits only bounded JSON edit operations naming existing tasks,
targets, paths, and typed values. Invalid paths, fields, targets, and values are
rejected. A separate Python process applies the edits and runs public tests; no
model-supplied program is imported or executed.

The five fillers change the actual evidence available to the model: symbol/path
retrieval, path/commit provenance, deployment/workspace routing, dependency
order, and public-test feedback. The model never sees a complete plan identifier
or hidden expected patch.

\subsection{Trace-induced maps}

A disjoint map-training split evaluates 25 candidates on four tasks with two
targets each, producing 200 typed target traces. Registered vertex fields are
selected from the task schema; filler-stage nomination maps are estimated from
the training labels; and the incidence graph is the maximum-weight spanning
tree plus the two strongest redundant edges. This yields six overlaps and
3,600 zero-error projection-composition checks.

Before confirmatory outcomes were opened, an independent reconstruction
enumerated all 25 structural boundaries on eight fresh tasks and 16 targets.
All 7,200 direct-versus-composed restriction coordinates agree. The resulting
cochain dimensions are reported in Section~\ref{sec:model}. This verifies
implementation functoriality for the declared fields; heldout calibration tests
whether those fields retain semantic information.

\subsection{First frozen equal-pool design}

The experiment uses the exact API identifiers
\url{deepseek-ai/DeepSeek-V4-Flash} and
\url{zai-org/GLM-5-FP8}; the operator attests that the latter serves GLM-5.2.
Four fresh screen--heldout task replicates are evaluated by both endpoints.
Every model--task block evaluates the same 25 candidates once, giving 200
candidate evaluations. Public-test failure permits at most one correction call,
so the final matrix contains 272 model calls.

The primary budget is candidate evaluations to the first bundle passing both
screen and heldout tasks. Secondary resource measures are calls to success and
provider-reported prompt, completion, and total tokens to success. Token totals
include correction calls and are compared between policies only within an exact
endpoint; they measure model-mediated token effort, not elapsed time. Other
metrics are joint heldout success, successful filler cost, non-target
regressions, screen-to-heldout regressions, obstruction calibration, and
structural computation.

The registered primary contrast is within-block adaptive-NSGA-II minus
full-class candidate evaluations, tested by an exact one-sided sign test. The
eight endpoint--task measurements share four task replicates, so we additionally
average endpoints within task before a conservative cluster sensitivity.

\subsection{Baseline search results}

Seventy-one of 200 candidate--block rows pass heldout tests, and all 71 are
exact joint successes: 35 for DeepSeek and 36 for GLM. These endpoint counts are
descriptive, not model-population estimates.

\begin{figure}[t]
\centering
\includegraphics[width=0.98\textwidth]{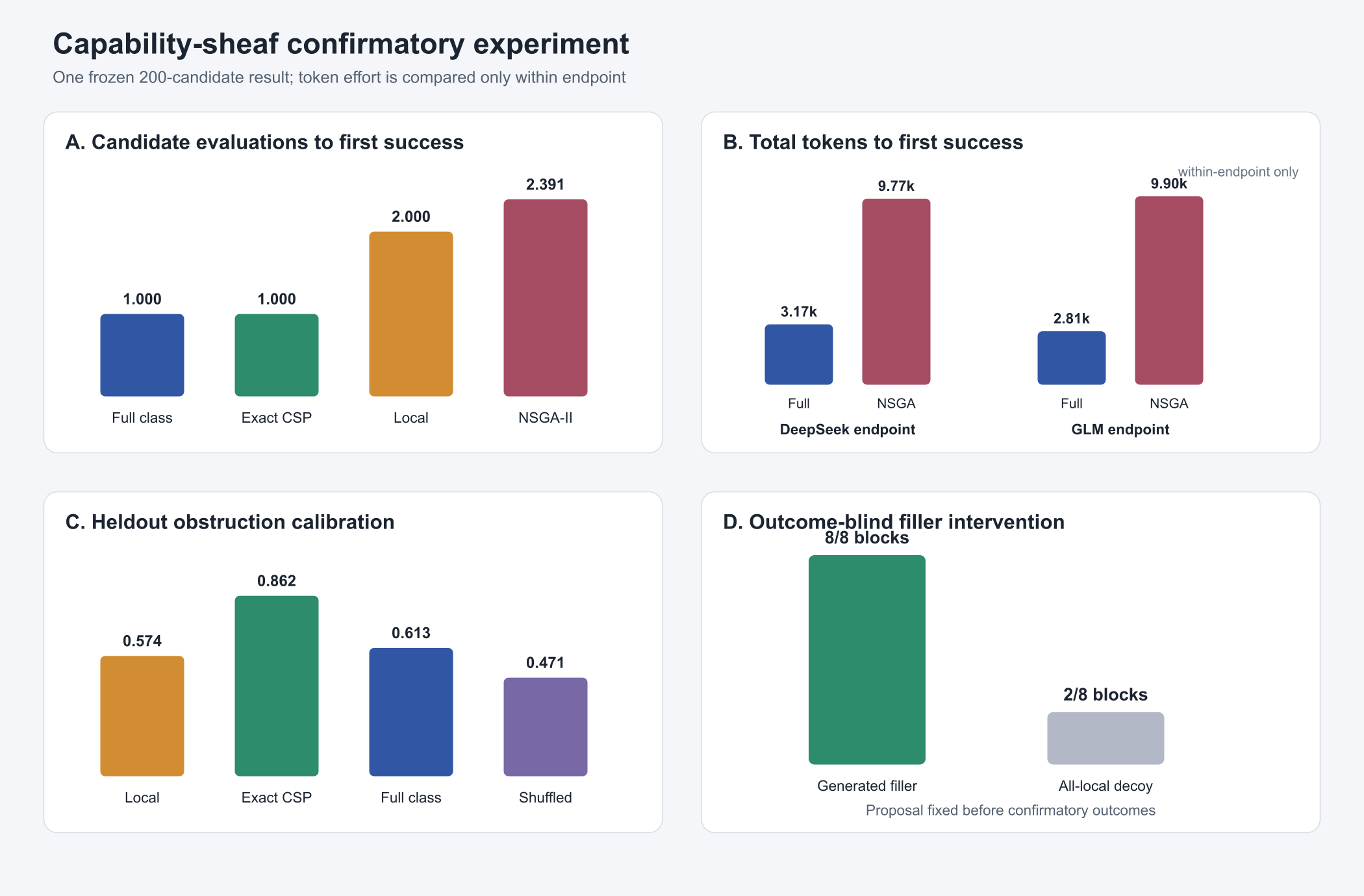}
\caption{One confirmatory experiment. Panel A reports the equal-pool stopping
budget. Panel B reports total tokens to first success separately inside each
endpoint. Panel C shows heldout balanced accuracy of structural criteria.
Panel D compares the outcome-blind generated filler with the registered
all-local structural decoy.}
\label{fig:experiment}
\end{figure}

\begin{table}[t]
\centering
\small
\caption{Equal-pool search. Candidate evaluations are primary; calls include
public-test corrections. Lower values, regressions, and filler cost are better.}
\label{tab:policy}
\begin{tabular}{lrrrrr}
\toprule
policy & success & evaluations & calls & regressions & filler cost\\
\midrule
full relative class & 1.000 & \textbf{1.000} & \textbf{1.000} & \textbf{0.000} & 6.000\\
exact CSP & 1.000 & \textbf{1.000} & \textbf{1.000} & \textbf{0.000} & 6.000\\
pairwise proxy & 1.000 & 2.000 & 3.000 & \textbf{0.000} & 6.000\\
local marginal & 1.000 & 2.000 & 3.000 & \textbf{0.000} & 6.000\\
shuffled class & 1.000 & 2.000 & 2.250 & \textbf{0.000} & \textbf{5.000}\\
adaptive NSGA-II & 1.000 & 2.391 & 2.922 & 3.156 & 5.633\\
random & 1.000 & 2.391 & 2.922 & 3.156 & 5.633\\
\bottomrule
\end{tabular}
\end{table}

For the primary contrast, seven NSGA-II-minus-full block differences equal
$1.25$ candidate evaluations and the remaining difference equals $2.375$.
All favor full class, with mean difference $1.390625$ and one-sided
$p=0.00390625$. Averaging the two
endpoints within each shared task gives differences
$(1.8125,1.25,1.25,1.25)$ and $p=0.0625$. Thus the direction holds in every
task and endpoint, while task-level replication remains small.

The outcome-blind cost-3 proposal passes exactly in all eight blocks with one
call per block and no heldout non-target regression. The registered all-local
decoy passes only two blocks. This closes the loop from typed obstruction,
through filler generation and recomputation, to executable heldout behavior.

\subsection{Token effort and calibration}

\begin{table}[t]
\centering
\small
\caption{Mean provider-reported total tokens to first success. Values are
policy comparisons within endpoint; absolute counts are not ranked across
different tokenizers.}
\label{tab:tokens}
\begin{tabular}{lrr}
\toprule
policy & DeepSeek endpoint & GLM endpoint\\
\midrule
full relative class & \textbf{3,172.5} & \textbf{2,808.5}\\
exact CSP & \textbf{3,172.5} & \textbf{2,808.5}\\
local marginal & 10,763.5 & 9,591.0\\
shuffled class & 3,626.5 & 11,664.0\\
adaptive NSGA-II & 9,770.2 & 9,901.5\\
\bottomrule
\end{tabular}
\end{table}

Relative to NSGA-II, full-class search reduces total tokens to success by
67.5\% for DeepSeek and 71.6\% for GLM; completion-token reductions are 63.2\%
and 68.1\%. Exact CSP has identical stopping candidates and therefore identical
token effort. The token result is consequently evidence for global relational
information, not for an incremental cohomological advantage.

Heldout balanced accuracy is 0.574 for local-only, 0.862 for exact CSP, 0.613
for full relative class, and 0.471 for shuffled class. Full class has
specificity 1.0 but sensitivity 0.225; exact CSP is the strongest compatibility
predictor. Structural feature computation takes a mean 4.45 ms per task
replicate on the artifact host, excluding the one-time human and data cost of
choosing the vocabulary and training the restriction map.

\FloatBarrier

\section{Hidden-Interior Quotient Experiment}
\label{sec:hidden-experiment}

\subsection{Frozen design and causal prediction}

We froze this experiment before making any new model call. It uses the same two
model endpoints and the same pool of 25 filler bundles as the baseline study.
The new matrix has 20 independent task clusters. Each cluster has a fresh
screen task, a fresh heldout task, and a separate model seed. Both endpoints
evaluate every candidate in every cluster. The final matrix therefore contains
1,000 candidate outcomes. Public-test feedback adds at most one correction
call, giving 1,360 model calls in total.

The structural score uses the subdivided complex from
Section~\ref{sec:hidden-quotient}. Every registered overlap coordinate has one
hidden mediator. We evaluate four outcome-blind interventions on that mediator:
stale, aligned-left, aligned-right, and randomized. In all 4,000 registered
candidate--target--intervention checks, the quotient signature is unchanged.
The quotient and stale-raw rankings differ at 24 of 25 positions in every task
cluster.

The preregistered primary policy minimizes quotient obstruction weight, then
uses interior-completion rate, exact-lift rate, nomination margin, and filler
cost. The stale-raw policy scores the two half-edge residuals before the
quotient and then uses local coverage and cost. This is a targeted stress test:
the stale mediator is distinct from both public endpoint values, so a raw score
can treat a compatible boundary as inconsistent. The aligned-left policy is
the negative control. In that condition, the raw score already tracks the
endpoint mismatch.

The independent statistical unit is the task cluster. For each cluster, we
average candidate evaluations to first joint success over the two endpoints.
The frozen primary contrast is stale-raw minus quotient. We use an exact
one-sided paired sign test over the 20 clusters. Exact CSP is the positive
semantic control. Pairwise, local, adaptive NSGA-II, and random search use the
same candidate outcomes. NSGA-II and random search each use four registered
search seeds per model--cluster block.

\subsection{Primary result and ablation}

The 1,000 outcomes contain 278 candidates that pass both screen and heldout
tests. The quotient policy reaches the first success in 1.000 candidate
evaluation in every model--cluster block. The stale-raw policy needs 2.000.
After averaging endpoints, every task cluster has paired difference $+1$.
Thus all 20 non-tied clusters favor the quotient. The exact one-sided sign-test
value is
\[
p=2^{-20}=9.5367\times10^{-7}.
\]
The mean paired difference is 1.000 candidate evaluation; the cluster
bootstrap 95\% interval is $[1.000,1.000]$. Equivalently, quotienting reduces
the candidate stopping budget by 50\% relative to the stale raw
representative.

\begin{figure}[t]
\centering
\includegraphics[width=0.99\textwidth]{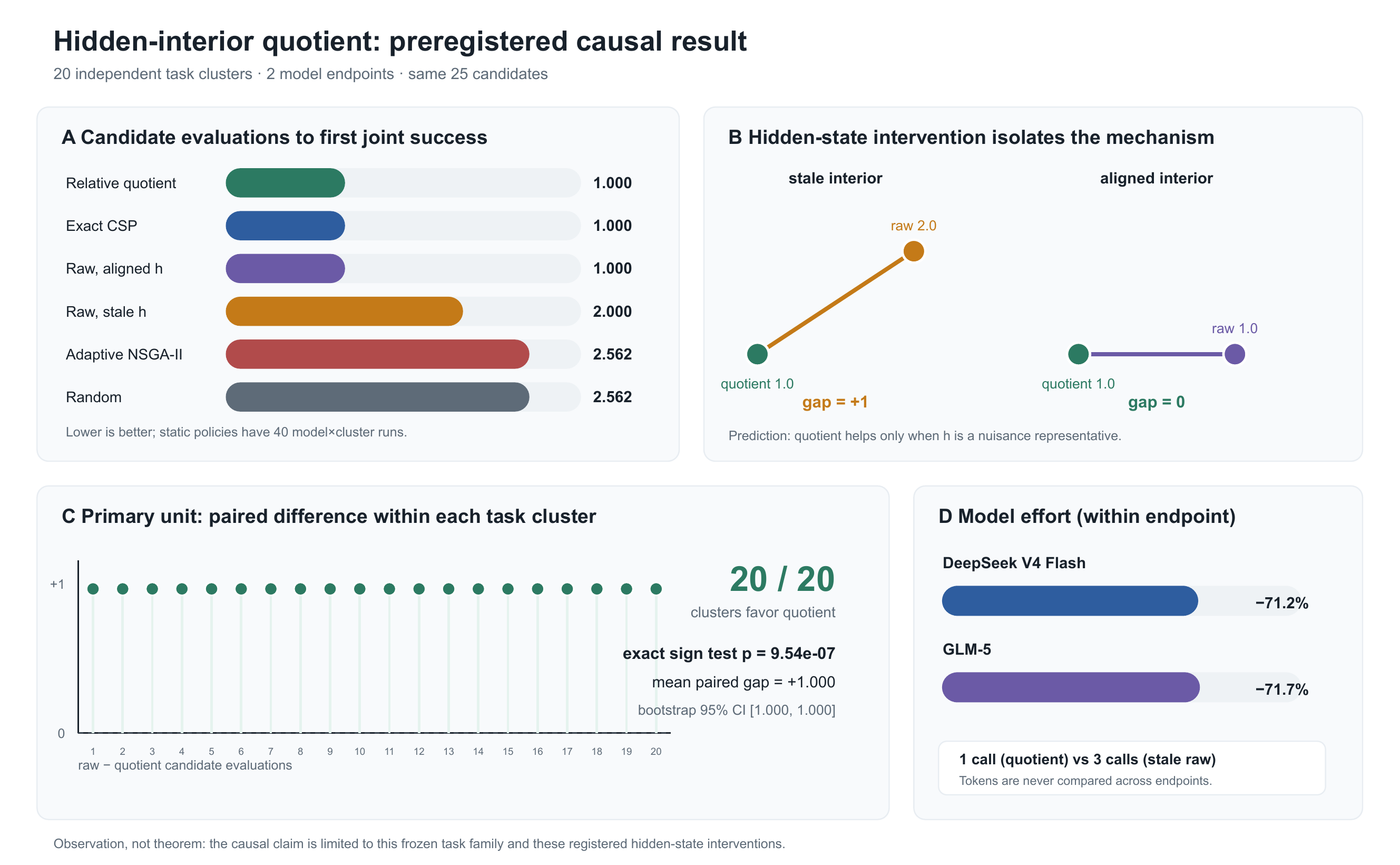}
\caption{Preregistered hidden-interior result. Panel A shows the equal-pool
stopping budget. Panel B is the causal check: the stale representative creates
a gap, while aligning the hidden state removes it. Panel C shows all 20 paired
task-cluster differences, not candidate-level pseudo-replicates. Panel D shows
total-token reductions within each endpoint. The figure is computational
evidence for this frozen task family, not a theorem about arbitrary harnesses.}
\label{fig:hidden-results}
\end{figure}

The aligned-left raw policy also reaches success in 1.000 evaluation. Its
paired difference from the quotient is zero in every cluster. This ablation is
important: it shows that the gain is tied to a nuisance interior
representative. It is not a generic reward for adding a cohomological label.
Exact CSP also matches the quotient in every cluster, as predicted.

\begin{table}[t]
\centering
\small
\caption{Hidden-interior experiment. All policies reach success. Candidate
evaluations are the primary budget. Token means are compared only within an
endpoint. Static policies have 40 model--cluster runs; NSGA-II and random have
160 runs because they use four search seeds.}
\label{tab:hidden-policy}
\begin{tabular}{lrrrr}
\toprule
policy & evaluations & calls & DeepSeek tokens & GLM tokens\\
\midrule
relative quotient & \textbf{1.000} & \textbf{1.000} & \textbf{3,050.8} & \textbf{2,706.5}\\
exact CSP & \textbf{1.000} & \textbf{1.000} & \textbf{3,050.8} & \textbf{2,706.5}\\
raw, aligned interior & \textbf{1.000} & \textbf{1.000} & 3,088.3 & 2,784.5\\
raw, stale interior & 2.000 & 3.000 & 10,587.3 & 9,548.4\\
pairwise proxy & 2.000 & 3.000 & 10,587.3 & 9,548.4\\
local marginal & 2.000 & 3.000 & 10,587.3 & 9,548.4\\
adaptive NSGA-II & 2.562 & 3.837 & 12,768.4 & 12,011.6\\
random & 2.562 & 3.837 & 12,768.4 & 12,011.6\\
\bottomrule
\end{tabular}
\end{table}

Within endpoint, the quotient reduces total tokens to success by 71.2\% for
DeepSeek and 71.7\% for GLM relative to the stale-raw policy. It reduces model
calls from 3.000 to 1.000. These are within-endpoint effort comparisons; they
do not rank the two model families. The equality with exact CSP sets the claim
boundary: the experiment shows value from quotienting hidden-state nuisance
over a fixed raw representative, not value beyond an equally informed exact
solver.

\section{Real-Repository Stress Test}
\label{sec:real}

\subsection{Benchmark, split, and outcome firewall}

We next test whether the controlled effect transfers to real patch fusion. We
use the seven-source SWE-bench Multilingual pool from PatchFuseBench
\citep{yang2026patchfusion,jimenez2024swebench}. The pool contains 300 GitHub
issues. We split the 41 eligible repositories, not individual issues: 20
repositories and 160 issues form the development split; 21 repositories and
140 issues form a sealed confirmatory split. This section uses only the
development split.

The 160 issues contain 1,120 possible source slots. Twenty-three slots are
missing, but no issue is removed. Exact duplicate patches are collapsed only
after their provenance is retained. The resulting pool has 875 unique
source-compatible patches. We decompose them against the base commits into
2,579 lossless edit atoms. Each atom keeps its file, base interval, old bytes,
new bytes, mode change, and source-patch support. The fixed conflict graph has
2,022 interval-overlap pairs and 109 source-local semantic-alternative pairs.

GLM-5 assesses each atom against the issue, a source-derived rubric, and the
local base context. It never receives benchmark tests, solved labels, or the
gold patch. A strong repository-held-out selector supplies the same anchor to
all methods. It resolves 127/160 issues. The matched router receives atom
evidence, candidate membership, conflicts, and the anchor. The relative router
receives the same inputs plus the signed relative block. The exact control
uses the same evidence in a deterministic binary optimization.

If a method returns an unchanged pool patch, we use its corrected
content-addressed verdict. Every new patch is run with the official pinned
SWE-bench harness. The final evaluation has 153/153 completed fresh-patch
rows and no infrastructure failures. Incomplete Docker calls are not counted
as failed tests.

\subsection{Identifiability counterexample and repair}

The first atom construction uses one full-pool matrix
\[
D:\bbR^{m}\longrightarrow\bbR^{p+q+2},
\]
where the columns are all available atoms, the first $p$ rows are issue
obligations, the next $q$ rows are conflict interfaces, and the final rows are
build and regression risk. For target $b$, the router receives the canonical
representative of $[b]$ in $\operatorname{coker}D$.

This class is global to the issue, not specific to a binary selection. For any
two selections $x_1,x_2$,
\begin{equation}
\label{eq:pool-invariance}
[b-Dx_1]-[b-Dx_2]=[D(x_2-x_1)]=0
\quad\text{in }\operatorname{coker}D.
\end{equation}
Thus the class cannot rank $x_1$ against $x_2$. The least-squares atom
potentials may still influence a language model, but they are not different
cohomology classes for different patches.

We repair this defect by indexing the complex by the candidate action. If
patch $S$ contains atom set $S$, let $D_S$ contain only those hidden interior
columns and the common obligation-and-risk rows. Its score is
\begin{equation}
\label{eq:candidate-quotient}
q(S)=
\frac{\lVert(I-D_SD_S^+)b\rVert_2}{\lVert b\rVert_2}.
\end{equation}
The matched control uses the same atom evidence but replaces the span and
projection by coordinate-wise maximum support and maximum risk. Exact
candidate scoring is a second control. Of 875 candidate complexes, 848 have
positive interior rank and positive quotient dimension. The value $q(S)$
varies across candidates on 120/160 issues.

\begin{figure}[t]
\centering
\includegraphics[width=0.99\textwidth]{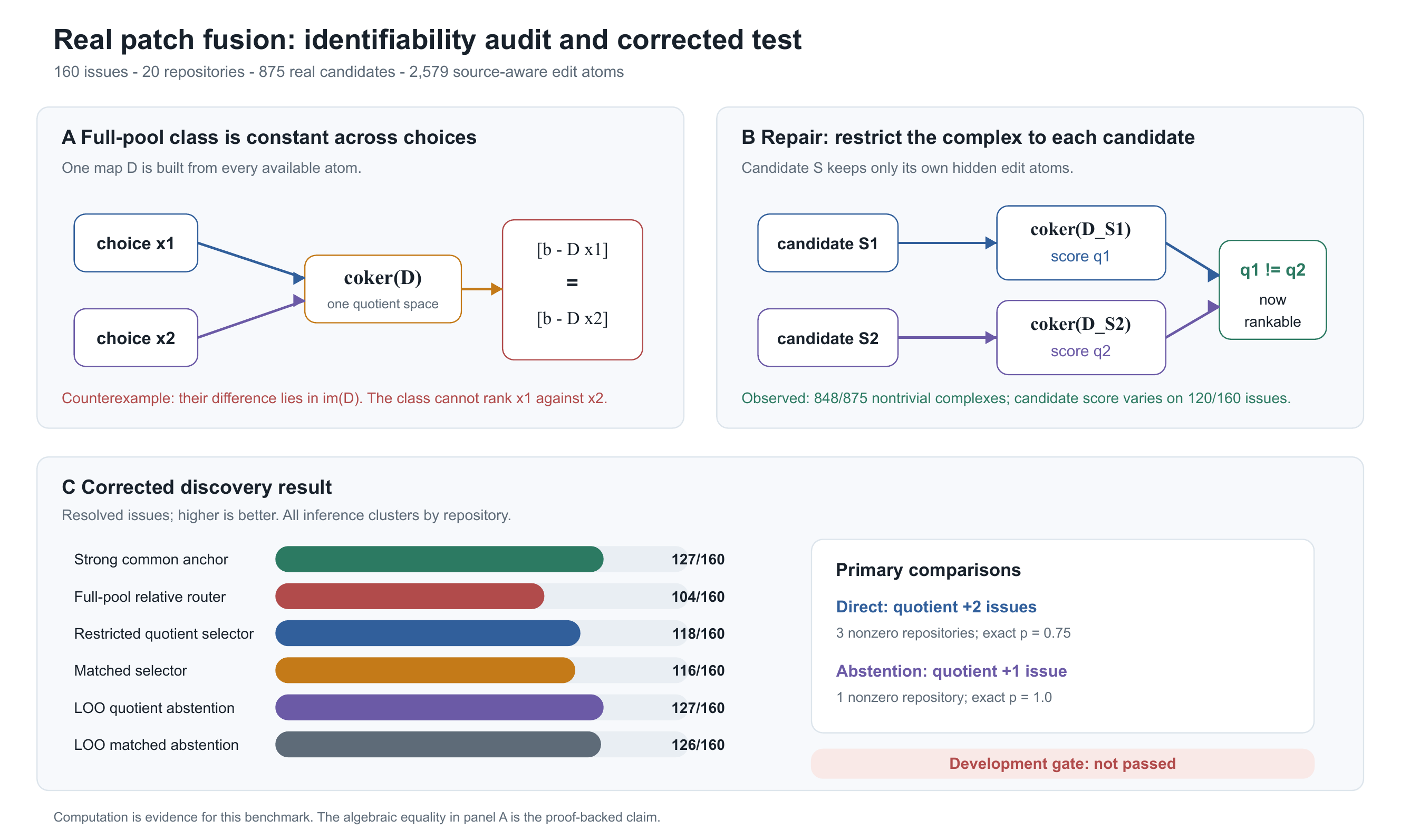}
\caption{Real patch-fusion audit. Panel A is a finite and matrix
counterexample: the full-pool class is identical for all selections; this is
the proof-backed statement in Equation~\eqref{eq:pool-invariance}. Panel B
shows the candidate-indexed repair. Its different scores are an algebraic
computation, not a correctness theorem. Panel C reports the empirical
development results. The corrected quotient is discriminative, but neither
direct selection nor abstention passes the repository-level gate.}
\label{fig:real-audit}
\end{figure}

\subsection{Results}

Table~\ref{tab:real-results} separates three questions. Full-pool routing asks
the model to construct one atom subset. Candidate-restricted selection chooses
one submitted patch. Repository-held-out abstention decides whether to keep
the strong anchor.

\begin{table}[t]
\centering
\small
\caption{Real-repository development results. ``Changed'' counts outputs that
differ from the common anchor. ``New'' counts patches absent from the submitted
pool. All denominators are 160 issues.}
\label{tab:real-results}
\begin{tabular}{P{0.20\textwidth}P{0.29\textwidth}rrr}
\toprule
stage & method & resolved & changed & new\\
\midrule
\multirow{4}{*}{\shortstack[l]{full-pool atom\\routing}}
 & common anchor & \textbf{127} & 0 & 0\\
 & matched router & 109 & 79 & 26\\
 & relative router & 104 & 85 & 37\\
 & exact semantic control & 96 & 120 & 112\\
\midrule
\multirow{4}{*}{\shortstack[l]{candidate-\\restricted}}
 & common anchor & \textbf{127} & 0 & 0\\
 & matched selector & 116 & 122 & 0\\
 & restricted quotient & 118 & 123 & 0\\
 & exact semantic selector & 119 & 122 & 0\\
\midrule
\multirow{3}{*}{LOO abstention}
 & common anchor & \textbf{127} & 0 & 0\\
 & matched gate & 126 & 6 & 0\\
 & quotient gate & \textbf{127} & 7 & 0\\
\bottomrule
\end{tabular}
\end{table}

The full-pool relative router resolves 104 issues and the matched router
resolves 109. Their paired difference is $-5/160$; the repository-macro
difference is $-0.0316$, the cluster bootstrap 95\% interval is
$[-0.0838,0.0102]$, and the exact repository sign-flip value is $p=0.28125$.
No new fused patch from any of the three methods turns an anchor failure into
a success. This rules out a real fusion gain for the present atom evidence and
router.

The candidate-indexed quotient is better than its matched control by two
issues: four paired wins and two losses. The effect occurs in only three
repositories and has exact $p=0.75$. It also remains nine issues below the
anchor. A symmetric leave-one-repository-out abstention procedure learns when
to keep the anchor. The quotient gate recovers 127/160, while the matched gate
gets 126/160. Their only discordant issue favors the quotient, but one nonzero
repository gives $p=1.0$.

The preregistered development gate requires at least four issue gains, a
positive repository-macro effect, six nonzero repositories, and
$p\leq0.2$. Neither corrected experiment passes. We therefore do not open the
21 confirmatory repositories. Further threshold searches on the same outcomes
would be post-selection bias.

\section{Claim Boundary and Threats to Validity}
\label{sec:limits}

The controlled studies use generated JSON, a bounded edit language, and compact
public tests. Their purpose is mechanism identification. They do not estimate
performance on natural repositories. The baseline also has only four
independent task clusters. The hidden-state study has 20 clusters, but its stale
mediator is a deliberate intervention, not an estimate of how often deployed
systems contain stale state.

The aligned ablation supports the controlled mechanism: when the nuisance
representative is removed, the quotient gap disappears. Exact CSP also matches
the quotient. Thus the controlled evidence supports invariance to hidden-state
choice, not superiority over an equally informed exact solver.

The real-repository study improves external validity but remains one fixed
candidate pool. Its seven sources, issue distribution, and strong anchor may
not represent other generation systems. Atom boundaries are lossless at the
diff level, but the obligation and risk evidence comes from one GLM-5 pass.
Independent GLM-5 and Qwen hunk studies earlier in development showed material
evidence-model dependence. Better semantic interfaces, generated tests, or
new edit bytes might change the result.

Repositories are the statistical clusters. Candidate patches, model sources,
and atom rows within a repository are not independent samples. The direct
candidate quotient differs from its matched control in only three repositories;
the abstention comparison differs in one. These designs cannot support a
small repository-level $p$-value, even though they include 160 issues.

The exact atom optimization is not computationally hard on this workload.
HiGHS solves all 2,579 binary atom variables, including an issue with 338
variables, in 3.45 seconds after presolve. We therefore make no speed claim
from the size of the naive binary grid. Likewise, quotient vanishing is not
binary semantic feasibility.

Provider-reported tokens in the controlled studies measure effort within one
endpoint's tokenizer and accounting rules. They do not rank model families.
The real study also excludes the human cost of defining rubrics, checking
source boundaries, and auditing harness failures.

The real development protocols were fixed before their corresponding scores
were joined to outcomes, but they followed earlier negative development
results. Their $p$-values are exploratory. The confirmatory repositories remain
sealed because the development gate failed. A positive claim requires a new
method fixed before those outcomes are opened.

\section{Related Work}
\label{sec:related}

Classical sheaf theory formalizes local data and gluing
\citep{bredon1997sheaf,maclane1992sheaves}; cellular sheaves and their Laplacians
provide finite computational models \citep{curry2014sheaves,hansen2019spectral}.
Sensor integration illustrates how local consistency can diagnose distributed
data \citep{robinson2017canonical}.  We use the same mathematical machinery for
typed agent behavior rather than geometric measurements.

Recent expository work surveys the path from finite-poset sheaves to machine
learning \citep{ayzenberg2025sheaf}, while a sheaf-theoretic account of tasks in
distributed systems treats solvability through compatible local views
\citep{felber2025tasks}.  Our task object is narrower: a frozen finite harness
candidate and its executable typed trace.

Knowledge sheaves express schema-constrained knowledge-graph embeddings as
approximate global sections \citep{gebhart2023knowledge}.  More directly,
\citet{olivieri2026theoryshift} rank finite scientific-theory transitions by
transport and gluing obstruction in a controlled AI-agent benchmark.  Our
setting differs in its object of intervention and validation: fillers change
repository operations, models synthesize bounded edits, public tests execute,
and search policies are compared behind a prospective outcome firewall.

The sheaf-theoretic account of contextuality characterizes locally consistent
families without global sections \citep{abramsky2011sheaf}.  Cohomological
witnesses and known false negatives motivate our exact-plus-linear hierarchy
\citep{abramsky2012cohomology,abramsky2017paradox,caru2017cohomology}.  Connections
between global sections and robust CSP further support the exact semantics
\citep{abramsky2013robust,oconghaile2022cohomology}.

Neural sheaf diffusion and connection-Laplacian networks learn sheaf-valued
representations on graphs \citep{bodnar2022neural,barbero2022connection}.  Our
object is instead an executable capability abstraction used to repair an agent
harness.  Persistent homology and stable vectorizations study multiscale
topological summaries \citep{carlsson2009topology,cohensteiner2007stability,
adams2017persistence}; persistence is complementary to, but not needed for, the
finite relative obstruction tested here.

Automated prompt, program, workflow, and agent optimization motivate the outer
loop \citep{khattab2023dspy,hu2025adas,zhang2025aflow,lee2026metaharness,
ursekar2026vero}.  HarnessFix is especially close in using failed trajectories
to localize harness flaws and validate scoped repairs \citep{chen2026harnessfix}.
Our contribution is not a new general optimizer.  It is a
structural observation channel and a typed repair ordering that can sit inside
such optimizers when local capability signatures are insufficient.

PatchFusion studies the same post-generation decision problem as our real
stress test \citep{yang2026patchfusion}. It uses deterministic repeated-atom
evidence and reports 236/300 solved issues on its SWE-bench Multilingual pool.
We reuse that fixed candidate pool but ask a different question: does a
hidden-state quotient add information beyond matched semantic evidence? Our
negative result does not contradict PatchFusion's deterministic fusion result.
It shows that the present cohomological construction does not improve it.

Table~\ref{tab:closest-work} makes the closest distinctions explicit
\citep{khattab2023dspy,hu2025adas,zhang2025aflow,lee2026metaharness,
ursekar2026vero,chen2026harnessfix,olivieri2026theoryshift,
yang2026patchfusion}.

\begin{table}[t]
\centering
\caption{Closest systems by intervention object, structural authority, and
validation.  ``Not declared'' means only that the cited work does not use the
exact sheaf/CSP object studied here.}
\label{tab:closest-work}
\begin{tabular}{P{0.16\textwidth}P{0.22\textwidth}P{0.27\textwidth}P{0.22\textwidth}}
\toprule
work & object changed or diagnosed & structural authority & validation\\
\midrule
DSPy / AFlow / ADAS & prompts, programs, workflows, agent designs & task objective and optimizer state; no declared capability sheaf & heldout task scores\\
Meta-Harness & executable harness code & proposer access to source, prior scores, and traces & application and agent benchmarks\\
VeRO & versioned target agents & versions, rewards, observations, and budget ledger & reproducible budget-controlled evaluation\\
HarnessFix & trajectory-localized harness flaws and patches & trace IR with provenance and control flow & scoped patch validation and heldout tests\\
PatchFusion & one patch from a fixed multi-agent pool & deterministic repeated edit-atom evidence & official repair benchmarks\\
Olivieri--Hern\'andez & scientific theory-transition candidates & chart transport and gluing obstruction & controlled transition-card benchmark\\
this work & harness fillers and fixed-pool patches & finite stalks, exact CSP, and hidden-state relative quotients & public tests and official SWE-bench execution\\
\bottomrule
\end{tabular}
\end{table}
\FloatBarrier

\section{Conclusion}

A capability sheaf separates local usefulness from shared-state agreement. Its
exact CSP is the semantic decision rule. A relative class can add an invariant
diagnostic, but it cannot replace exact feasibility or execution.

The controlled hidden-state experiment works as intended. Quotienting a stale
interior representative reduces the candidate budget from 2.000 to 1.000 in
all 20 clusters. Aligning the hidden state removes the gap, and exact CSP
matches the quotient. This is evidence for an invariance mechanism.

The real-repository stress test gives a different result. The first full-pool
class is identical for every atom selection and therefore cannot rank them.
The candidate-indexed repair fixes this mathematical defect and produces
nontrivial, varying scores on most of the real pool. It gives a small advantage
over its matched selector, but the effect is not distributed across enough
repositories. Repository-held-out abstention ties the strong anchor rather
than improving it.

We therefore do not claim a real-world cohomological advantage. The main
scientific result is a boundary: hidden-state quotienting succeeds in the
controlled intervention, while the tested obligation-and-risk complex is too
weak for real patch fusion. A stronger construction must encode semantic
interfaces that actually glue across edits, rather than attach one global
class to an issue. Such a method should be frozen before the sealed
confirmatory repositories are opened.

\acks{The accompanying artifact contains the controlled protocols, all 1,200
controlled metadata/outcome pairs, the real
discovery split and content-addressed candidate pool, 153 completed fresh-patch
evaluations, exact tests, source hashes, regenerated figures, and the original
hash-bound runtime. No endpoint credential is included.}

\appendix

\section{Reproducibility Checklist}

The standalone repository contains:
\begin{itemize}
  \item \path{src/capability_sheaves}: task generation, exact CSP, hidden-state
  and candidate-indexed quotient construction, repository-clustered analysis,
  and figure rendering;
  \item \path{configs}: controlled protocols, real-benchmark development
  protocols, restriction maps, and revealed task-generation secrets;
  \item \path{data/raw}: 200 baseline metadata/outcome pairs with complete
  provider usage and hash bindings;
  \item \path{data/latent_quotient/raw}: 1,000 hidden-state metadata/outcome
  pairs, materialized before the corresponding external outcomes;
  \item \path{data/map_training}: the disjoint trace split used to construct the
  restriction map;
  \item \path{data/processed}: the frozen aggregate and independent cochain
  reconstruction;
  \item \path{results/swebench_real}: the discovery split, label and harness
  audits, atom inventory, source-aware evidence, fresh-patch checkpoints,
  candidate-indexed analyses, and publication figure;
  \item \path{results}: baseline outputs plus the frozen hidden-state plan,
  analysis, report, and other publication figures;
  \item \path{provenance}: the immutable 200-job plan and byte-identical
  hash-bound execution runtime; and
  \item \path{tests}: scientific, token-accounting, figure, and provenance
  invariants.
\end{itemize}

The readable package is a post-execution refactor. Original machine identifiers
remain only under \path{provenance} where changing them would invalidate frozen
hashes. They do not denote additional empirical studies in this paper.

\section{Additional Counterexamples}

\paragraph{Compatibility without local existence.}
If a required stage has no selected section, every registered equality among
the remaining stages may hold.  A pairwise all-ones vector therefore does not
imply Theorem~\ref{thm:gluing}'s local premise.

\paragraph{Linear vanishing without an executable candidate.}
The signatures $(1,0)$ and $(0,1)$ span $(1,1)$ over $\bbF_2$, but no member of
the portfolio realizes $(1,1)$.  This is why exact candidate membership follows
the quotient computation.

\paragraph{One pool-level class for every configuration.}
Fix the full atom map $D$ and target $b$. For any binary selections $x_1,x_2$,
the representatives $b-Dx_1$ and $b-Dx_2$ differ by
$D(x_2-x_1)\in\im D$. They therefore define the same class in
$\operatorname{coker}D$. A configuration-specific comparison must change the
admissible interior map, for example by using $D_{S_1}$ and $D_{S_2}$.

\paragraph{Small perturbation and rank inflation.}
For $D=\operatorname{diag}(1,0)$ and $b=e_2$, the target is distance one from
$\im D$.  The matrix
$\widehat D=\operatorname{diag}(1,\varepsilon)$ is arbitrarily close to $D$ but
has full image, making the untruncated residual zero.  Registered-rank
truncation in Theorem~\ref{thm:stability} is not optional.

\bibliography{references}

\end{document}